\documentclass[letterpaper, 10 pt, conference]{ieeeconf}  

\IEEEoverridecommandlockouts                              

\usepackage[usenames, dvipsnames]{color}
\usepackage{listings}
\usepackage{xcolor}
\usepackage{url}
\usepackage{float}
\usepackage{amsmath,amssymb,amsthm}
\usepackage{mathtools} 
\usepackage[english]{babel}

\usepackage{epsfig}
\usepackage{epstopdf}
\usepackage{subcaption}
\usepackage{multirow}
\usepackage{graphicx}

\usepackage[export]{adjustbox}
\usepackage[boxruled,vlined,linesnumbered]{algorithm2e}

\usepackage{hyperref}
\usepackage[autostyle, english = american]{csquotes}
\MakeOuterQuote{"}

\newcommand{\T}{{\mbox{\tiny \sf T}}}
\renewcommand{\T}{{\mathsf T}}

\newtheorem{theorem}{Theorem}
\newtheorem{lemma}[theorem]{Lemma}
\newtheorem{corollary}[theorem]{Corollary}
\newcommand{\dint}{\mathrm{d}}

\title{\LARGE \bf
	TIE: Time-Informed Exploration for Robot Motion Planning
}

\author{Sagar Suhas Joshi$^{1}$~ Seth Hutchinson$^{2}$ ~ Panagiotis Tsiotras$^{3}$
	\thanks{$^{1,2,3}$ Institute for Robotics and Intelligent Machines, Georgia Institute of Technology, USA.
	Email: {\small \{sagarsjoshi94, seth, tsiotras\}@gatech.edu}}
}

\begin{document}
	
	\maketitle
	\thispagestyle{empty}
	\pagestyle{empty}

	\begin{abstract}
		
		Anytime sampling-based methods are an attractive technique for solving kino-dynamic motion planning problems.
		These algorithms scale well to higher dimensions and can efficiently handle state and control constraints.
		However, an intelligent exploration strategy is required to accelerate their convergence and avoid redundant computations.
		Using ideas from reachability analysis, this work defines a "Time-Informed Set", that focuses the search for time-optimal kino-dynamic planning after an initial solution is found. 
		Such a Time-Informed Set includes all trajectories that can potentially improve the current best solution and hence
		exploration outside this set is redundant. 
		Benchmarking experiments show that an exploration strategy based on the TIS can accelerate the convergence of sampling-based kino-dynamic motion planners.
		
	\end{abstract}
	
\section{INTRODUCTION}

Sampling-based motion planners incrementally build a connectivity graph by generating random samples in the search-space. Popular algorithms such as RRT~\cite{lavalle2001randomized} can solve challenging problems in higher-dimensional spaces, but can only ensure probabilistic completeness. The RRT* algorithm \cite{karaman2011sampling} combines the exploration procedure in RRT with a "local rewiring" module to guarantee asymptotic optimality. 
Algorithms such as RRT$^{\#}$~\cite{arslan2013use}, FMT*~\cite{janson2015fast} and BIT*~\cite{gammell2015batch} use heuristics along with dynamic programming ideas to achieve faster convergence than RRT*. 

The "geometric" versions of the above sampling-based algorithms ignore kino-dynamic constraints of the robot and connect any two points in a Euclidean search space with a straight line. 
However, a general kino-dynamic problem requires the solution of a two-point boundary value problem (TPBVP), also called the "local steering" problem, for optimally connecting any two states. 
Karaman and Frazzoli extended the RRT* algorithm for kino-dynamic planning by incorporating such steering functions in \cite{karaman2010optimal}.    
Perez et al~\cite{perez2012lqr} linearized the system dynamics and solved the infinite-horizon linear quadratic regulator (LQR) problem to obtain a locally optimal steering procedure.
The kino-dynamic RRT* algorithm \cite{webb2013kinodynamic} penalizes the control effort and the trajectory duration while connecting any two states. 
The authors of \cite{webb2013kinodynamic} solve a fixed final state, free final time, optimal control problem for linear time invariant (LTI) systems to derive a steering function. A kino-dynamic version of FMT* is presented in \cite{schmerling2015optimal}.
Note that these algorithms rely on the availability of a local steering module to ensure asymptotic optimality. 
However, developing such computationally efficient TPBVP solvers may not be possible for many cases.
The GR-FMT algorithm \cite{hwan2015optimal} proposes a local steering method based on polynomial basis functions and segmentation for controllable linear systems.
The recently introduced Stable Sparse RRT (SST) and SST* \cite{li2015sparse} algorithms guarantee asymptotic optimality, while having access only to a forward propagation model of the system's dynamics. This eliminates the need for TPBVP solvers.
The SST procedure promotes the propagation of states with good path costs and performs a selective pruning operation to keep the number of stored nodes small.
\begin{figure}
	\centering
	\includegraphics[width=0.7\columnwidth]{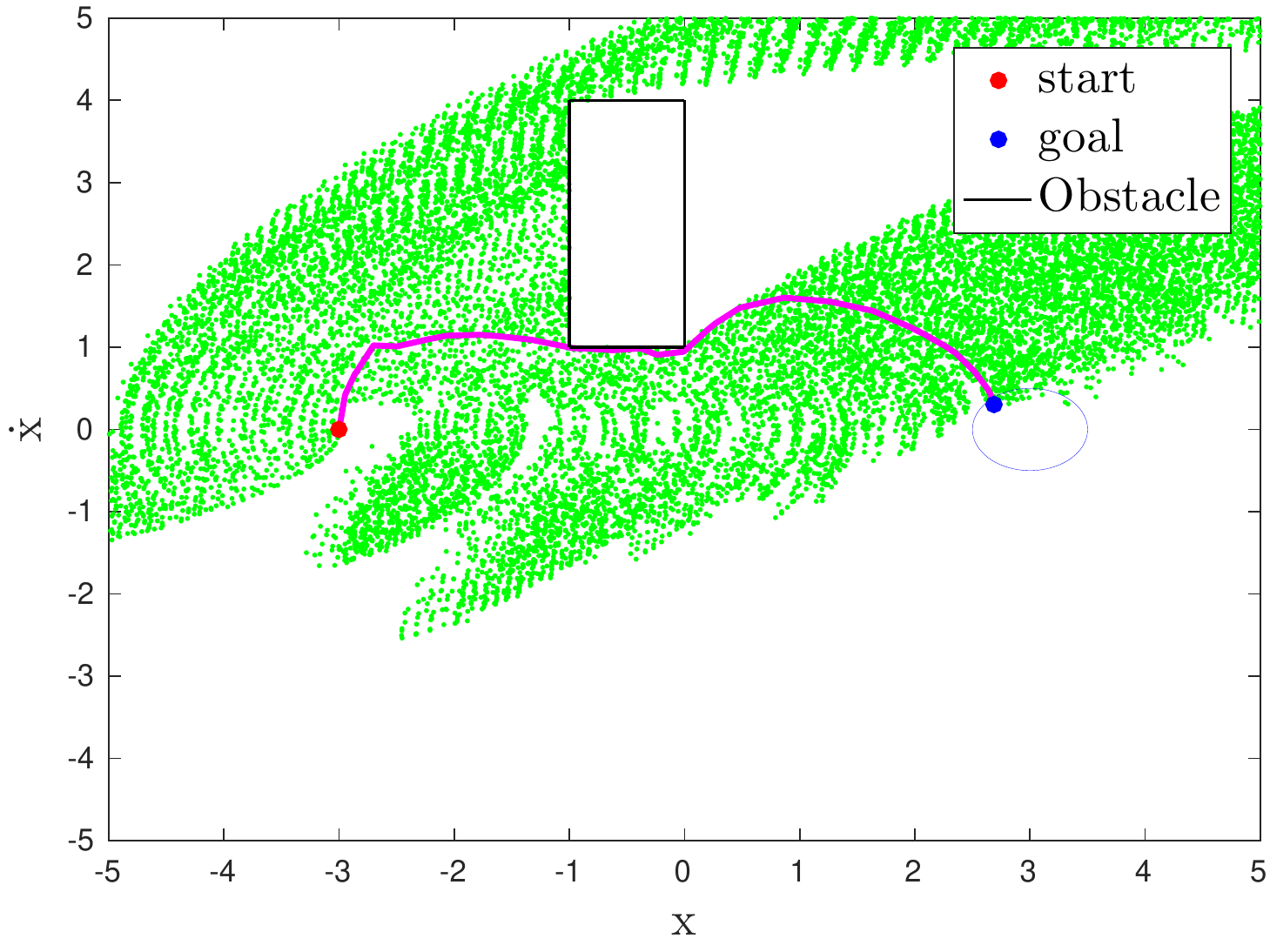}
	\includegraphics[width=0.7\columnwidth]{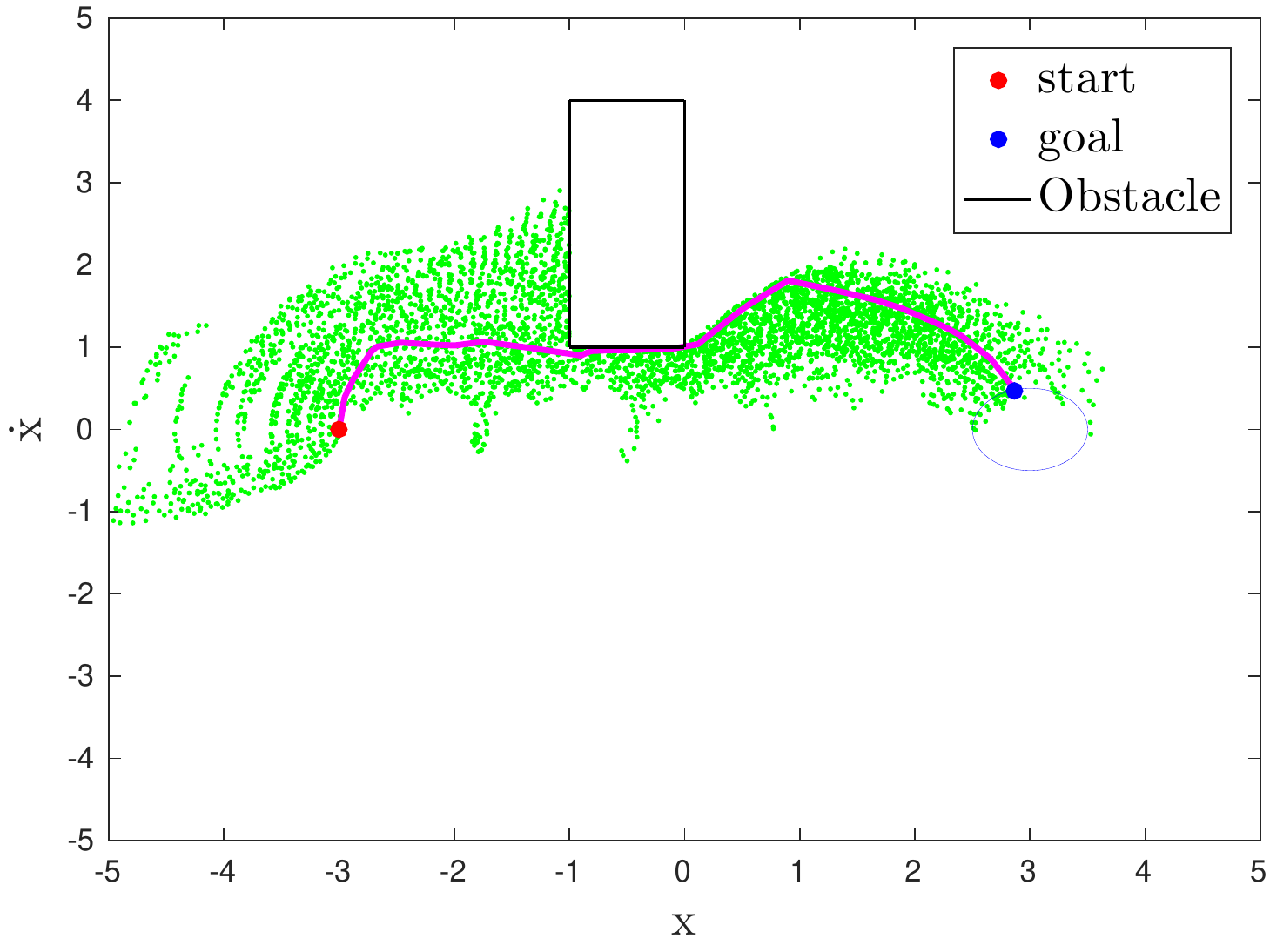}
	\caption{Time-optimal planning for a 2D system using the SST algorithm with uniform exploration (top) and the proposed strategy (bottom). The tree vertices generated are represented in green. Using the proposed strategy leads to a focused search.}
	\label{fig:toy2d}
\end{figure}

While significant progress has been made in the area of sampling-based kino-dynamic planners, developing intelligent exploration strategies to complement them still remains a challenging problem. 
Uniform random sampling results in a rapid exploration of the search-space and is effective for finding a first solution.
However, after an initial solution is found, exploration can be focused on a subset of the search-space that can potentially further improve the current solution.
For the case of geometric, length-optimal planning, Gammell et al~\cite{gammell2018informed} introduced the ``$L_2$-Informed Set'' that contains all the points that can \textit{potentially} improve the current solution. This set is a prolate hyper-spheroid with focii at the start and the goal states and its transverse diameter is equal to the current best solution cost. 
The direct Informed Sampling (IS) technique proposed in~\cite{gammell2018informed} provides a scalable approach to focus search, and shows dramatic convergence improvements in higher dimensions compared to the other state-of-the-art heuristic methods. 

However, as discussed in \cite{kunz2016hierarchical}, \cite{yi2018generalizing} deriving a parameterized representation or direct sampling of such Informed Sets for systems with differential constraints is a challenging problem. 
In this work, we propose an analogue to the Informed Set for the case of \textit{time-optimal} kino-dynamic planning using ideas from reachability analysis \cite{bansal2017hamilton,kurzhanskiy2010computation}.

Given a feasible (but perhaps sub-optimal) solution trajectory with time cost $T>0$, we define a Time-Informed Set (TIS) as the set that contains all the trajectories with time cost less than or equal to $T$. 
The planner can thus avoid redundant exploration outside the TIS.
The proposed exploration algorithm can be applied to a variety of systems, even if a tractable TPBVP solver may not be available. 

\section{RELATED WORK}

Prior work on intelligent exploration, such as \cite{akgun2011sampling,gammell2018informed,arslan2015machine,ichter2018learning} utilized heuristics and ideas from deep learning to improve the performance of sampling-based planners.  
The Informed SST (iSST) algorithm~\cite{littlefield2018informed} also leverages heuristics to guide search for kino-dynamic planning. DIRT~\cite{littlefield2018efficient} uses dominance informed regions along with heuristics to balance exploration and exploitation. 
However, iSST and DIRT may be ineffective in focusing the search for the cases where a good heuristic function is unavailable.

Concepts from reachability analysis have also been used for guiding exploration in sampling-based kino-dynamic planning.
Shkolnik et al~\cite{shkolnik2009reachability} used reachable sets in their RG-RRT algorithm to shape the Voronoi bias so as to find a feasible solution quickly.    
A discretized representation of the reachable space is proposed in \cite{pendleton2017numerical} to be used for sampling and nearest neighbor search. 
Chiang et al~\cite{chiang2019rl} trained an obstacle-aware time-to-reach (TTR) reachability estimator network to guide the RRT search process. 
However, the above techniques do not focus search on a subset of the search space based on current solution cost, which can lead to redundant exploration.

The algorithms proposed in~\cite{kunz2016hierarchical} and \cite{yi2018generalizing} are most relevant to the current work, as they address the problem of Informed Sampling for kino-dynamic motion planning.
Kunz et al~\cite{kunz2016hierarchical} proposed a hierarchical rejection sampling (HRS) method to generate informed samples for higher-dimensional systems. HRS essentially is a "bottom up" procedure that generates samples along the individual dimensions and combines them.
An accept/reject decision is taken for each partial sample until a complete sample in the informed set is generated. 
Yi et al~\cite{yi2018generalizing} proposed a Hit-and-Run Markov Chain Monte-Carlo (HNR-MCMC) algorithm to improve the sampling efficiency compared to HRS.
Given a previous sample in the Informed Set, the HNR-MCMC first samples a random direction and then uses rejection sampling to find the largest step-size so that the new sample lies inside the Informed Set. 
However, both HRS and HNR-MCMC assume availability of a local steering function, that gives the optimal cost (or a good under-estimate) connecting any two states. 
For minimum time problems, the above two methods can only be applied to specific systems, such as the double integrator. In this work, we address this issue by using ideas from reachability analysis to define the TIS. The proposed algorithm can thus be applied to a wide variety of systems.  

In the following sections, the time-optimal kino-dynamic motion planning problem is first defined, followed by the definition of the TIS and some theoretical results. 
The proposed exploration algorithm is then delineated along with some results from a series of numerical experiments. 

\section{PROBLEM DEFINITION}

Let $\mathcal{X} \subset \mathbb{R}^n $, $n \geq 2$ and $\mathcal{U} \subset \mathbb{R}^m $, $m \geq 1$  be compact sets representing the state and admissible control spaces respectively. Let $\mathcal{X}_\mathrm{obs} \subset \mathcal{X}$ denote the obstacle space and  $ \mathcal{X}_\mathrm{free}= \mathrm{cl}(\mathcal{X}\setminus \mathcal{X}_\mathrm{obs}) $ denote the free space.
Here, $\mathrm{cl}(S)$ represents the closure of the set $S \subset \mathbb{R}^n$.  
Let $\lambda( S )$ denote the Lebesgue measure of the set $S \subset \mathbb{R}^n$.
Let $\textbf{x}_\mathrm{s} \in \mathcal{X}_\mathrm{free}$ denote the initial state and let $\mathcal{X}_\mathrm{g} \subset \mathcal{X}_\mathrm{free}$ represent the goal set.
The time-optimal motion planning problem can be defined as follows:
\begin{subequations}
	\label{eq:TimeOptimalProblem}
	\begin{align}
	T^*= & \min_{\textbf{u}} ~ T \\
	\text{subject to:} & ~ \dot{\textbf{x}}(t) = f(\textbf{x}(t),\textbf{u}(t)), \label{eq:dynamics}\\
	& ~ \textbf{x}(0)=\textbf{x}_\mathrm{s}, ~\textbf{x}(T)\in \mathcal{X}_\mathrm{g}, \\
	& \textbf{x}(t) \in \mathcal{X}_\mathrm{free},~ \textbf{u}(t) \in \mathcal{U} ~~~ \text{for all} ~ t \in [0,T]. \label{eq:feasibility}
	\end{align}
\end{subequations}
Sampling-based algorithms solve the above problem by incrementally building a tree $\mathcal{T}=(V,E)$ that encodes the connectivity between a finite set of vertices $V \subset \mathcal{X}_\mathrm{free}$ with edges $E \subseteq V \times V$. 
The trajectory and the cost representing an edge are calculated either by using a steering function or by forward propagation of the system model using random controls.
\begin{figure*}
	\centering
	\includegraphics[width=0.65\columnwidth]{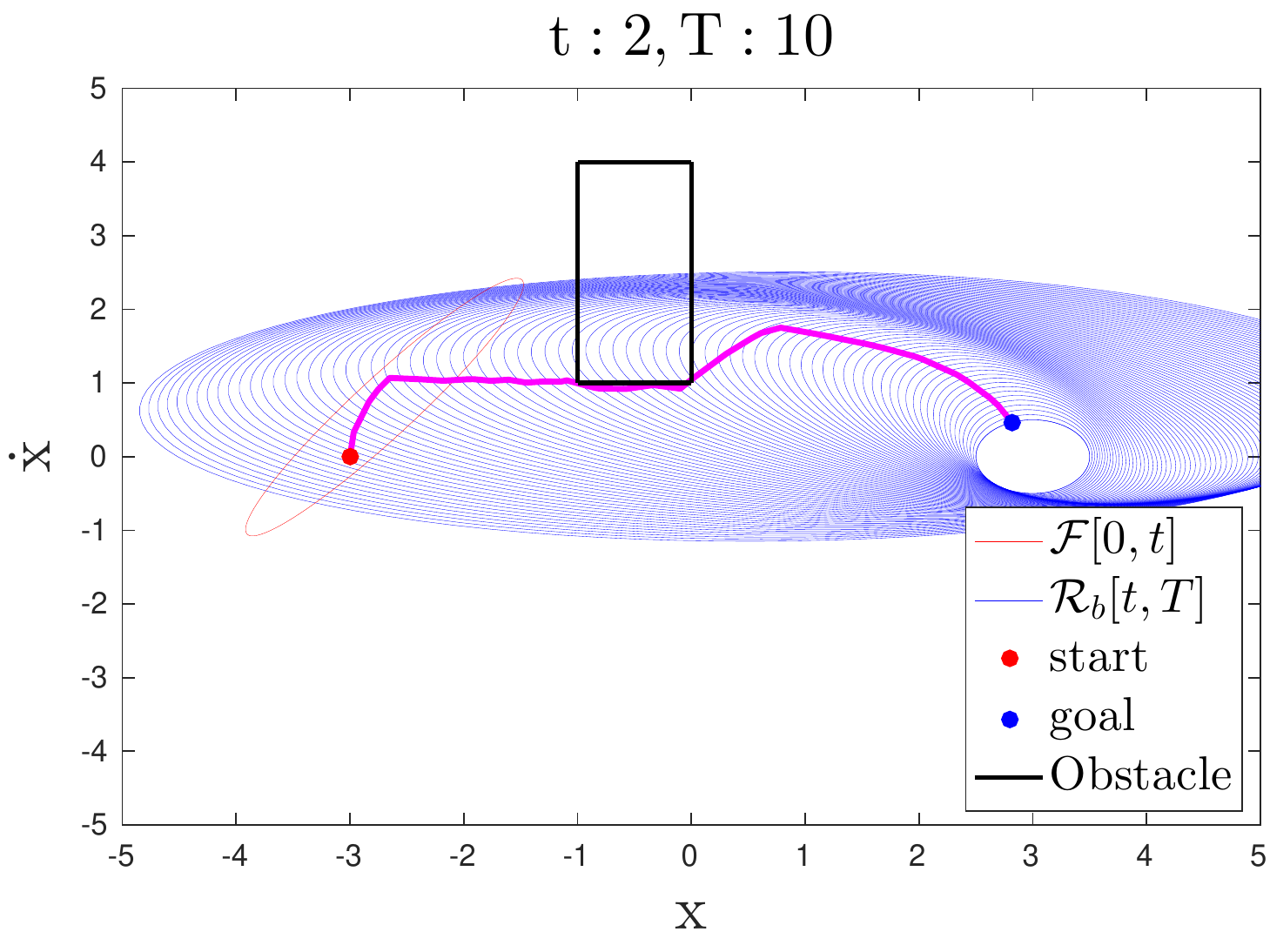}
	\includegraphics[width=0.65\columnwidth]{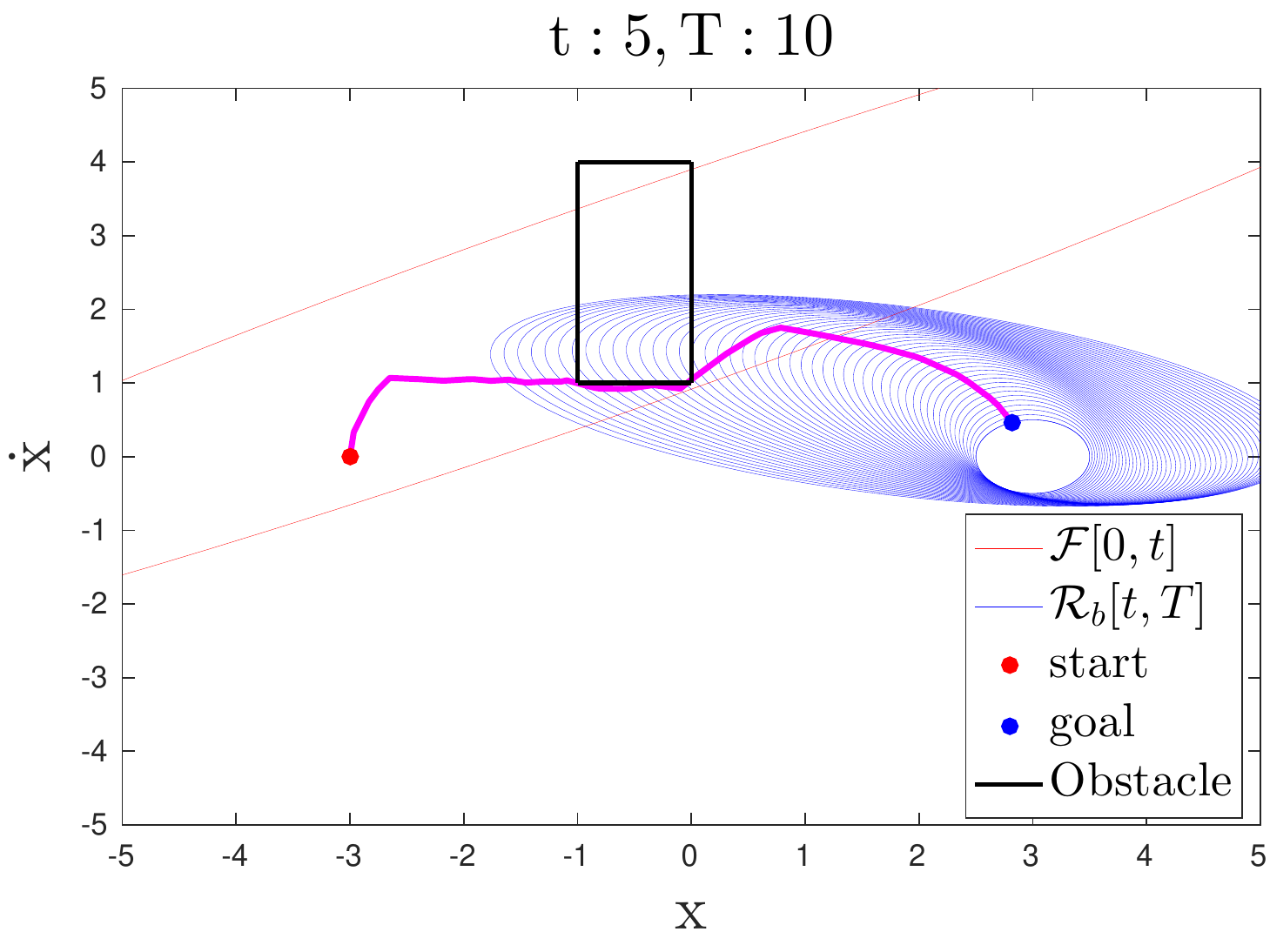}
	\includegraphics[width=0.65\columnwidth]{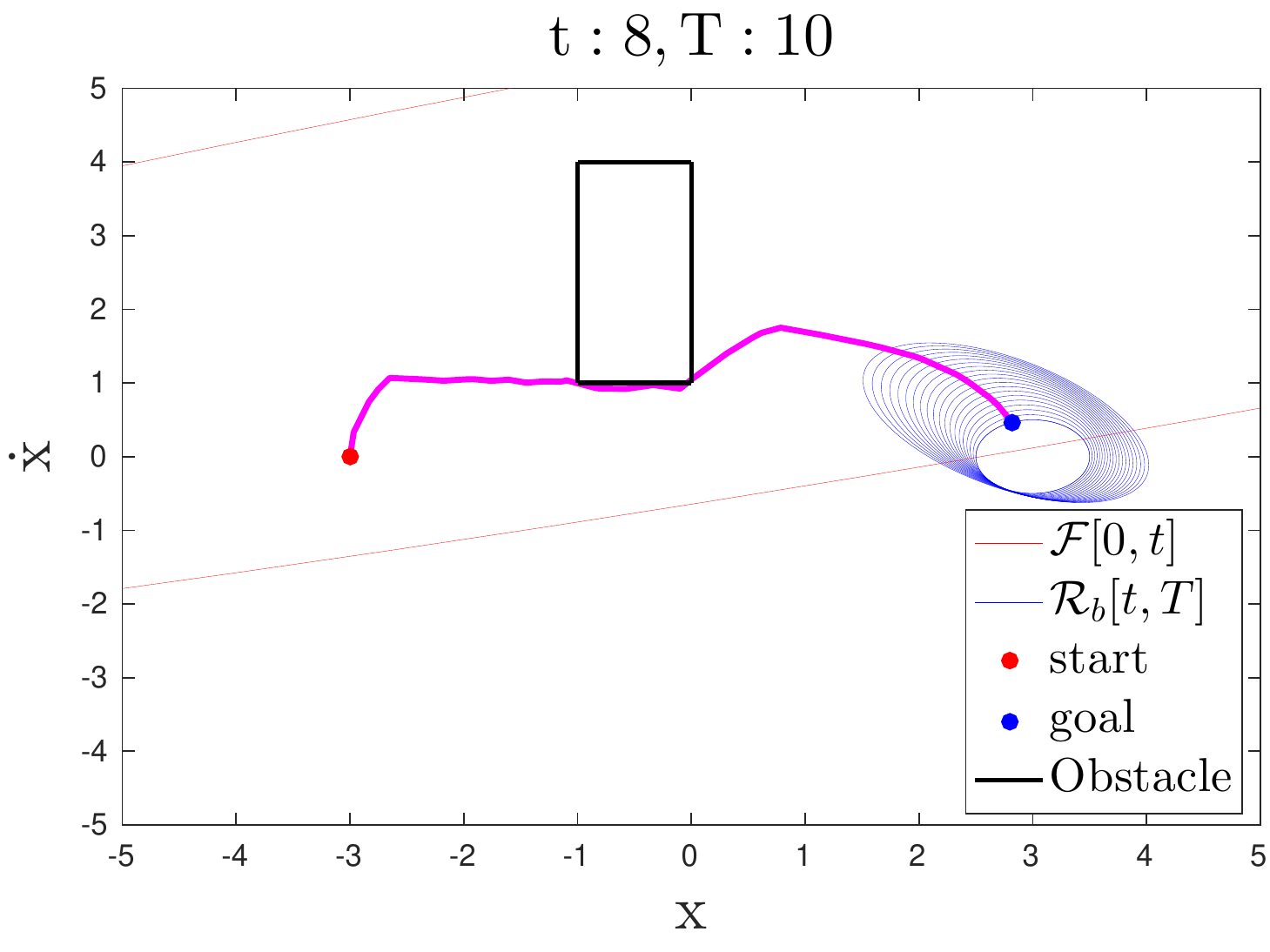}
	\caption{Evolution of the forward reachable set $\mathcal{F}[0,t]$ and the backward reachable tube $\mathcal{R}_b[t,T]$ for the 2D Toy system at time $t=2,5,8$. 
	Note that $\Omega(T)$ comprises of the intersections $\mathcal{F}[0,t] \cap \mathcal{R}_b[t,T]$.}
	\label{fig:toy2d_reachsets}
\end{figure*}

\section{TIME-INFORMED SET}

Consider the set of points that can be reached at time $t$, starting from $\textbf{x}_\mathrm{s}$ at time $t_0<t$, using admissible controls,
\begin{equation}
\label{eq:forwardreachsetatt}
\begin{aligned}
\mathcal{X}_f[t_0,t]=\{\textbf{z}\in \mathcal{X} ~ | ~ \exists ~ \textbf{u}:[t_0,t] \rightarrow \mathcal{U},~ \textbf{x}:[t_0,t] \rightarrow \mathcal{X},\\
\text{s.t} ~  \textbf{x}(t_0)=\textbf{x}_\mathrm{s},\textbf{x}(t)=\textbf{z}, \dot{\textbf{x}}(t) = f(\textbf{x}(t),\textbf{u}(t)) \}.
\end{aligned}
\end{equation}
Let $\mathcal{F}[t_0,t]$ be an over-approximation of $\mathcal{X}_f[t_0,t]$, i.e., $\mathcal{X}_f[t_0,t] \subseteq \mathcal{F}[t_0,t]$.
Similarly, the set of points starting at time $t$ that can reach $\mathcal{X}_\mathrm{g}$ at time $t_f>t$ using admissible controls can be defined as,
\begin{equation}
\label{eq:backwardreachsetatt}
\begin{aligned}
\mathcal{X}_b[t,t_f]=\{\textbf{z}\in \mathcal{X} ~ | ~ \exists ~\textbf{u}:[t,t_f] \rightarrow \mathcal{U},~ \textbf{x}:[t,t_f] \rightarrow \mathcal{X},\\
\text{s.t} ~ \textbf{x}(t)=\textbf{z},\textbf{x}(t_f)\in \mathcal{X}_\mathrm{g}, \dot{\textbf{x}}(t) = f(\textbf{x}(t),\textbf{u}(t)) \}.
\end{aligned}
\end{equation}
Let $\mathcal{B}[t,t_f]$ be an over-approximation of $\mathcal{X}_b[t,t_f]$, i.e., $\mathcal{X}_b[t,t_f] \subseteq \mathcal{B}[t,t_f]$. Note that state constraints ensuring collision-free trajectories are not imposed while defining the above sets.
The (over-approximated) backward reachability \textit{tube} over the interval $[t,t_f]$ includes the set of all points starting at time $t$, that can reach $\mathcal{X}_\mathrm{g}$ at any time $\tau \in [t,t_f]$
\begin{equation}
\label{eq:backwardreachableset}
\mathcal{R}_b[t,t_f]=\bigcup_{t \leq \tau \leq t_f} \mathcal{B}[t,\tau].
\end{equation}
Assume that a feasible (perhaps sub-optimal) solution to problem (\ref{eq:TimeOptimalProblem}) with time cost $T>0$ is available. Consider the following definition of the Time-Informed Set (TIS)
\begin{equation}
\label{eq:omegaT}
\Omega(T)=\bigcup_{0 \leq t \leq T} \mathcal{F}[0,t] \cap \mathcal{R}_b[t,T].
\end{equation}
Intuitively, $\Omega(T)$ contains all the points $\textbf{x} \in \mathcal{X}$ that can be reached from $\textbf{x}_\mathrm{s}$ at a time $t$, where $0 \leq t \leq T$, i.e., $\textbf{x} \in \mathcal{F}[0,t]$ and then can reach the goal at time $\tau$, $t\leq \tau \leq T$, i.e., $\textbf{x} \in \mathcal{R}_b[t,T]$.
Please see Fig.~\ref{fig:toy2d_reachsets} and the attached video\footnote{\url{https://www.youtube.com/watch?v=dnMHb7uFEGw}} for a visualization of $\Omega(T)$.

The following theoretical arguments formally prove that given a sub-optimal solution with time cost $T $, the set $\Omega(T)$ contains all the trajectories with time cost $T$ or less.
\begin{lemma}
	\label{lemma:ftintersectbt}
	Given a feasible solution with cost $T>0$, $\mathcal{F}[0,t] \cap \mathcal{B}[t,T] \neq \varnothing$ for all $t \in  [0,T]$.
\end{lemma}

\begin{proof}
	Consider the solution trajectory with time cost $T$, $\zeta:[0,T] \rightarrow \mathcal{X}$, where $\zeta(0)=\textbf{x}_\mathrm{s}$ and $\zeta(T)=\textbf{x}_\mathrm{g}$.
	For any point $\textbf{x}$ on this trajectory, there exists $ t \in [0,T]$ such that $\textbf{x}=\zeta(t)$. 
	Thus, $\textbf{x} \in \mathcal{F}[0,t]$ and $\textbf{x} \in \mathcal{B}[t,T]$. It follows that, $\textbf{x} \in \mathcal{F}[0,t] \cap \mathcal{B}[t,T]$. Therefore, $\mathcal{F}[0,t] \cap \mathcal{B}[t,T] \neq \varnothing$.
\end{proof}
\begin{lemma}
	\label{lemma:rbsubset}
	$\mathcal{R}_b[t,T_1] \subset \mathcal{R}_b[t,T_2]$ for any $T_2>T_1>t>0$.
\end{lemma}

\begin{proof}
	Note from the definition (\ref{eq:backwardreachableset}),
	\begin{equation*}
	\begin{aligned}
	\mathcal{R}_b[t,T_2]&=\bigcup_{t \leq \tau \leq T_2} \mathcal{B}[t,\tau]\\
	&= \bigg( \bigcup_{t \leq \tau \leq T_1} \mathcal{B}[t,\tau] \bigg) \bigcup \bigg( \bigcup_{T_1\leq \tau \leq T_2} \mathcal{B}[t,\tau] \bigg).
	\end{aligned}	
	\end{equation*}	
	Since $\mathcal{R}_b[t,T_1] = \bigcup_{t \leq \tau \leq T_1} \mathcal{B}[t,\tau] $ it follows that $\mathcal{R}_b[t,T_1] \subset \mathcal{R}_b[t,T_2]$.
\end{proof}
\begin{theorem}
	\label{theorem:omegacontiansT}
	The set $\Omega(T)$ contains all trajectories with time cost exactly $T$. 
\end{theorem}
\begin{proof}
	Consider any solution trajectory $\zeta:[0,T] \rightarrow \mathcal{X}$ with time cost $T>0$, where $\zeta(0)=\textbf{x}_\mathrm{s}$, $\zeta(T)=\textbf{x}_\mathrm{g}$. For any point $\textbf{x}$ on this trajectory, there exists $ t \in [0,T]$ such that $\textbf{x}=\zeta(t)$. Then, $\textbf{x} \in \mathcal{F}[0,t]$ and $\textbf{x} \in \mathcal{B}[t,T]$. 
	This implies that $\textbf{x} \in \mathcal{F}[0,t] \cap \mathcal{B}[t,T]$ and hence $\textbf{x} \in \mathcal{F}[0,t] \cap \mathcal{R}_b[t,T]$. Thus, $ \textbf{x} \in \Omega(T)$. Since $t$ is arbitrary, if follows that $\zeta (t) \in \Omega(T)$ for all $t \in [0,T]$. 
\end{proof}
\begin{theorem}
	\label{theorem:omegaincreasing}
	$ \Omega(T_1) \subset \Omega(T_2)$ for any $T_2>T_1>0$.
\end{theorem}

\begin{proof}
Recall that the set $\Omega(T_2)$ is defined by
	\begin{equation}
	\Omega(T_2)=\bigcup_{0 \leq t \leq T_2} \mathcal{F}[0,t] \cap \mathcal{R}_b[t,T_2],
	\end{equation}	
which can be re-written as
	\begin{equation*}
	\begin{aligned}
\Omega(T_2)	&=\bigg( \bigcup_{0 \leq t \leq T_1}\mathcal{F}[0,t] \cap \mathcal{R}_b[t,T_2] \bigg)
	\bigcup\\
	&\bigg( \bigcup_{T_1 \leq t  \leq T_2} \mathcal{F}[0,t] \cap \mathcal{R}_b[t,T_2] \bigg).
	\end{aligned}
	\end{equation*}	
	From Lemma~\ref{lemma:rbsubset}, it follows that $\mathcal{R}_b[t,T_1] \subset \mathcal{R}_b[t,T_2]$.
	Hence, $\Omega(T_1)=\bigcup_{0 \leq t \leq T_1}\mathcal{F}[0,t] \cap \mathcal{R}_b[t,T_1] \subset \bigcup_{0 \leq t \leq T_1}\mathcal{F}[0,t] \cap \mathcal{R}_b[t,T_2]$.
	Thus, $ \Omega(T_1) \subset \Omega(T_2)$.
\end{proof}
\begin{corollary}
	\label{theorem:omegacontainsT}
	Given a solution to (\ref{eq:TimeOptimalProblem}) with time cost $T$, the set $\Omega(T)$ defined in (\ref{eq:omegaT}) contains all the trajectories with cost less than or equal to $T$. 
	Conversely, any trajectory that is not contained inside $\Omega(T)$ has time cost $T^{'} >T$
\end{corollary}

\begin{proof}
	From Theorem \ref{theorem:omegacontiansT}, it follows that $\Omega(T)$ contains all trajectories with time cost exactly $T$. 
	Theorem \ref{theorem:omegaincreasing} implies that $\Omega(T)$ is a superset of all the sets containing trajectories with time cost less than $T$. 
	Thus, $\Omega(T)$ also contains all the trajectories with cost less than or equal to $T$.
\end{proof}

After a, perhaps sub-optimal, solution with cost $T$ is found, any state that lies on an improved solution path necessarily lies inside the TIS.
The search can thus be focused onto the TIS.  
This can avoid redundant computations and accelerate convergence, especially for higher dimensional problems.

\section{TIME-INFORMED vs $L_2$-INFORMED SET}

This section examines the relationship between the TIS defined in (\ref{eq:omegaT}) and the $L_2$-Informed Set from~\cite{gammell2018informed} for a special case of a linear single integrator system. 
The purpose of this investigation is to show that the TIS is a generalization of the $L_2$-Informed Set approach in \cite{gammell2018informed}.
Consider the case of single-integrator dynamics $\dot{\textbf{x}}(t)=\textbf{u}(t)$, for which
\begin{equation}
\label{eq:singleIntReachSets}
\begin{aligned}
\mathcal{F}[t_0,t]&=\{\textbf{x}\in \mathcal{X} \ | \|\textbf{x}-\textbf{x}_\mathrm{s}\|_2 \leq  u_\mathrm{max}(t-t_0)  \}\\
\mathcal{B}[t,t_f]&=\{\textbf{x}\in \mathcal{X} \ | \|\textbf{x}-\textbf{x}_\mathrm{g}\|_2 \leq u_\mathrm{max}(t_f-t) \}.
\end{aligned}
\end{equation}
Here, $\|.\|_2$ represents the $L_2$-norm.
As the set $\mathcal{U}$ is compact, there exists a $u_\mathrm{max}>0$, so that $\|\textbf{u}(t)\|_2\leq u_\mathrm{max}$ for all $t$.
Note that, for this special case, the forward and backward reachable sets defined in (\ref{eq:singleIntReachSets}) are concentric circles.  
Then, for a given $t<t_f$, we have $\mathcal{B}[t,t_f]=\mathcal{R}_b[t,t_f]$ and hence 
$\Omega(T)=\bigcup_{0 \leq t \leq T} \mathcal{F}[0,t] \cap \mathcal{B}[t,T]$. 
Thus, for any $\textbf{x} \in \Omega(T)$, we have $\|\textbf{x}-\textbf{x}_\mathrm{s}\|_2 \leq  u_\mathrm{max}t $ and $\|\textbf{x}-\textbf{x}_\mathrm{g}\|_2 \leq u_\mathrm{max}(T-t)$. 
Adding the two inequalities we get, 
\begin{equation}
\label{eq:l2TimeInfSet}
\Omega(T)= \{\textbf{x}\in \mathcal{X} \ | \ \|\textbf{x}-\textbf{x}_\mathrm{s}\|_2 + \|\textbf{x}-\textbf{x}_\mathrm{g}\|_2 \leq u_\mathrm{max}T \}
\end{equation}
The TIS in (\ref{eq:l2TimeInfSet}) in this case has the same prolate hyper-spheroid form as the $L_2$-Informed Set \cite{gammell2018informed}.
Thus, the TIS can be seen as a generalization of the $L_2$-Informed Set.

\section{TIME-INFORMED EXPLORATION}

Although obtaining the exact reachable sets defined in (\ref{eq:forwardreachsetatt}), (\ref{eq:backwardreachsetatt}) may not be computationally tractable, various techniques have been proposed to obtain tight over-approximations of these sets. 
These include application of polytopes and zonotopes~\cite{girard2006efficient}, ellipsoidal calculus~\cite{kurzhanskiy2010computation} and formulating reachability problem as a Hamilton-Jacobi-Bellman (HJB) PDE~\cite{bansal2017hamilton}.  
In this work, we use the ellipsoidal technique which provides a scalable framework for reachability analysis of robots with linear-affine dynamics. 
However, as discussed  later on, the HJB reachability formulation can be used to extend the algorithms proposed in this work for general cost-functions and non-linear systems. 

Consider the special case of linear kino-dynamic systems. Concretely, the constraint (\ref{eq:dynamics}) is $\dot{\textbf{x}}(t) = A\textbf{x}(t)+B\textbf{u}(t)$, with $A \in \mathbb{R}^{n \times n}$, $B \in \mathbb{R}^{n \times m}$.
Then, $\mathcal{X}_f[t_0,t]$ and $\mathcal{X}_b[t,t_f]$ can be defined as
\begin{equation}
\label{eq:fwdreachsetattlinear}
\begin{aligned}
\mathcal{X}_f[t_0,t]=\{\textbf{x}\in \mathcal{X} ~ | ~ \exists & ~ \textbf{u}:[t_0,t] \rightarrow \mathcal{U},\text{s.t} \\
\textbf{x}=\mathrm{e}^{A(t-t_0)}\textbf{x}_\mathrm{s}+ & \int_{t_0}^t \mathrm{e}^{A(t-\tau)}B\textbf{u}(\tau) ~ \dint \tau  \},\\
\mathcal{X}_b[t,t_f]=\{\textbf{x}\in \mathcal{X} ~ | ~ \exists & ~\textbf{u}:[t,t_f] \rightarrow \mathcal{U},\text{s.t} \\   \textbf{x}=\mathrm{e}^{-A(t_f-t)}\textbf{x}_\mathrm{g}-&\int_t^{t_f} \mathrm{e}^{-A(\tau-t)}B\textbf{u}(\tau) ~ \dint \tau  \}.
\end{aligned}
\end{equation}
Here, $\textbf{x}_\mathrm{g} \in \mathcal{X}_\mathrm{goal}$. 
A hyper-sphere over-approximation to the above sets can be constructed as follows \cite{girard2006efficient},
\begin{equation}
\label{eq:sphoverapproxattlinear}
\begin{aligned}
\mathcal{F}[t_0,t]&=\{\textbf{x}\in \mathcal{X} \ | \|\textbf{x}-\mathrm{e}^{A(t-t_0)}\textbf{x}_\mathrm{s}\|_2 \leq  r(t_0,t,u_\mathrm{max})  \}, \\
\mathcal{B}[t,t_f]&=\{\textbf{x}\in \mathcal{X} \ | \|\textbf{x}-\mathrm{e}^{-A(t_f-t)}\textbf{x}_\mathrm{g}\|_2 \leq r(t,t_f,u_\mathrm{max}) \}, \\
r(t_1,t_2,& u_\mathrm{max})=(\mathrm{e}^{\|A\|_2(t_2-t_1)}-1) {\|B\|u_\mathrm{max}}/{\|A\|_2}.  
\end{aligned}
\end{equation}
Here, $\|M\|_2$ represents the induced two norm (maximum singular value) for a matrix $M$. 
However, the above over-approximation might be too conservative for the current application. 
See Fig.~\ref{fig:reachset_compare}. 
If the reachable sets are overtly conservative and $\lambda(\Omega(T)) \approx \lambda(\mathcal{X})$, then TIE may result in little or no focus of the search.
\begin{figure}
	\centering
	\includegraphics[width=0.6\columnwidth]{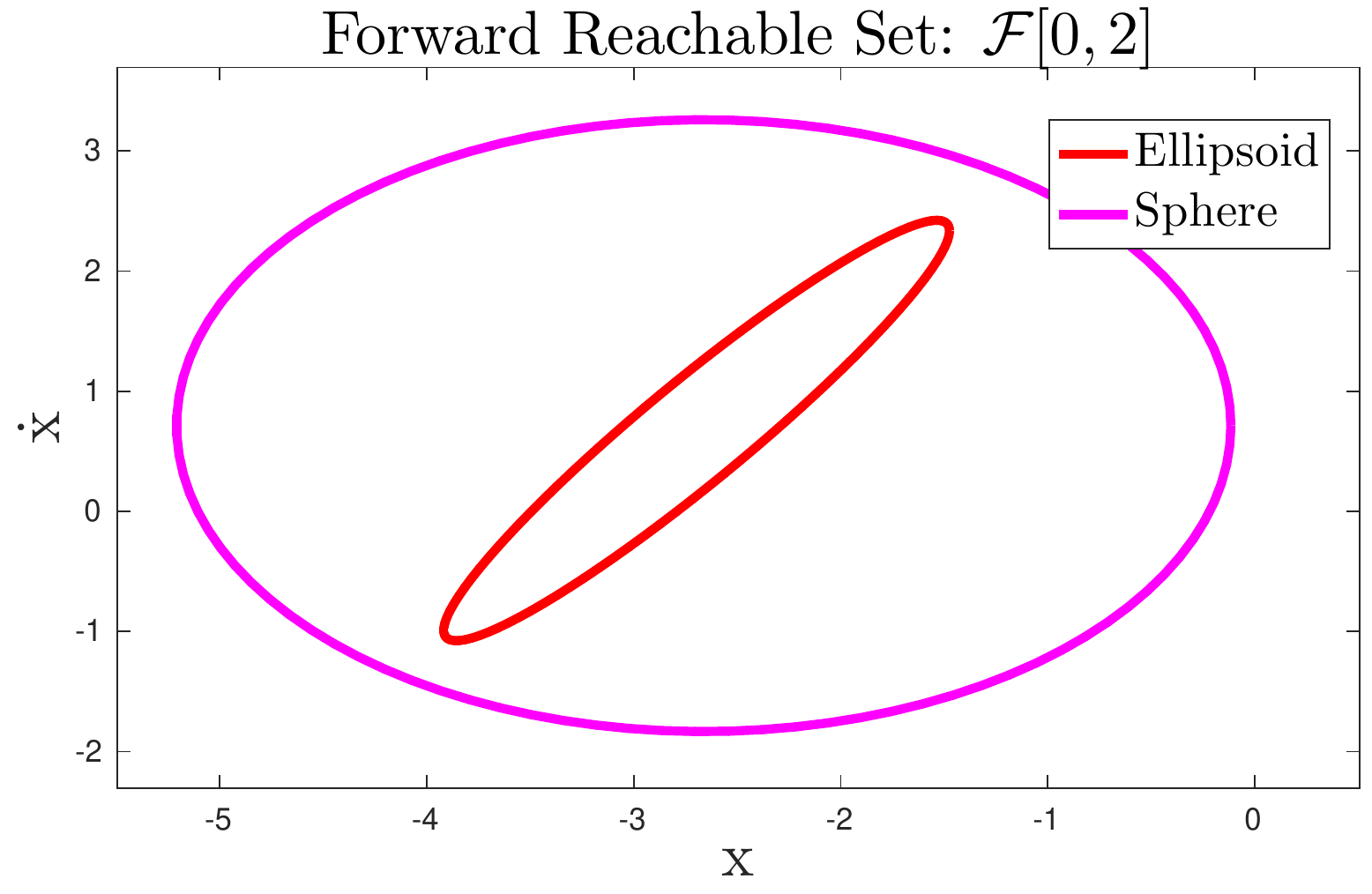}	
	\caption{Comparing the forward reachable set $\mathcal{F}[0,t]$ for the 2D system at $t=2$ using the hyper-sphere and ellipsoidal approximation.}
	\label{fig:reachset_compare}
\end{figure}
In contrast, the ellipsoidal technique~\cite{kurzhanskiy2010computation} approximates the reachable sets as ellipsoids, 
\begin{equation}
\mathcal{E}(\textbf{x}_c,Q)=\{\textbf{x} \in \mathbb{R}^n | \langle \textbf{x}-\textbf{x}_c,Q^{-1}(\textbf{x}-\textbf{x}_c) \rangle \leq 1  \}. 
\end{equation}
Here, $\textbf{x}_c$ is the center and $Q$ is the positive definite shape matrix of the ellipsoid.   
Forward and backward reachable sets, $\mathcal{F}[0,t], \mathcal{B}[t,T]$ can be obtained by solving an ordinary differential equation (ODE) for the center and shape matrix. 
Please see the Ellipsoidal Toolbox\footnote{\url{http://systemanalysisdpt-cmc-msu.github.io/ellipsoids/doc/main_manual.html}} documentation for a brief overview.
Note that the boundary conditions for the forward and backward reachable set ODE are the start and goal ellipsoids respectively.
From the problem definition in (\ref{eq:TimeOptimalProblem}), the start ellipsoid is encoded as a hyper-sphere with negligible radius around the center $\textbf{x}_\mathrm{s}$. 
The goal set $\mathcal{X}_\mathrm{g}$ is represented also as a hyper-sphere with a set radius around a center $\textbf{x}_\mathrm{g} \in \mathcal{X}_\mathrm{g}$.   
The ODE for the shape matrix can be solved and stored off-line. 
An analytical solution for the ODE describing the center's  trajectory can also be constructed. 
Thus, a "library" of reachable sets $\mathcal{F}[0,t], \mathcal{B}[t,T]$ can be created off-line to be used in the sampling and vertex inclusion algorithm described below. 
This library stores the value of center vector $\textbf{x}_c$ and matrices $Q$, $L$ of the forward and backward reachable sets.
Here $L$ is obtained using the Cholesky decomposition of $Q$, $Q=LL^\mathsf{T}$ and is used for generating samples inside $\mathcal{E}(\textbf{x}_c,Q$) \cite{gammell2018informed}.
Please see Fig.~\ref{fig:toy2d_reachsets} for a visualization of $\mathcal{F}[0,t]$ and $\mathcal{B}[t,T]$ constructed using the ellipsoid technique.

\subsection{Sampling Algorithm}

Algorithm \ref{alg:generateSample} describes a procedure to generate a new sample $\textbf{x}_\mathrm{rand}$ in $\Omega(T)$.
Notice from (\ref{eq:omegaT}) that $\Omega(T)$ consists of a union over the intersections of sets. Devising a direct sampling technique to generate uniform random samples in $\Omega(T)$ (as done for the $L_2$ Informed Set in \cite{gammell2018informed}) is hence a challenging task.
The proposed algorithm proceeds by first sampling a time $t$ in the interval $(0,T)$ according to a probability distribution $p_{[0,T]}(t)$ (line 2).
Ideally, to generate uniform random samples in $\Omega(T)$ with respect to the Lebesgue measure, this distribution needs to be $p_{[0,T]}(t)=\lambda(\mathcal{F}[0,t] \cap \mathcal{R}_b[t,T])/\lambda(\Omega(T))$. However, calculating and sampling from this distribution may not be tractable for general higher dimensional systems. %
Hence, for the sake of simplicity, we choose $p_{[0,T]}(t)$ to be uniform over the interval $[0,T]$.
Given $t$, the sets $\mathcal{F}[0,t],\mathcal{B}[t,T]$ can then be obtained from the library of stored reachable sets as discussed in the previous section.
We leverage the fact that $\mathcal{F}[0,t] \cap \mathcal{B}[t,T] \neq \varnothing$ from Lemma \ref{lemma:ftintersectbt} to generate a $\textbf{x}_\mathrm{rand} \in \mathcal{F}[0,t] \cap \mathcal{B}[t,T]$.
If the Lebesgue measure of $\mathcal{F}[0,t]$ is less than $\mathcal{B}[t,T]$, a uniform sample is generated in $\mathcal{F}[0,t]$ and checked if it belongs to $\mathcal{B}[t,T]$, otherwise, $\mathcal{B}[t,T]$ is sampled and checked if it belongs to $\mathcal{F}[0,t]$ (lines 4-13).
Notice from Fig.~\ref{fig:toy2d_reachsets} that $\lambda(\mathcal{F}[0,t])$ increases and $\lambda(\mathcal{B}[t,T])$ decreases as $t$ varies from $0$ to $T$. 
An efficient algorithm for generating uniform samples inside a hyper-ellipsoid is discussed in \cite{gammell2018informed}.
If no $\textbf{x}_\mathrm{rand} \in \mathcal{F}[0,t] \cap \mathcal{B}[t,T]$ can be generated in $n_s$ attempts, the algorithm returns a uniform random sample from the search-space $\mathcal{X}$ (line 14-15).

\IncMargin{.5em}
\begin{algorithm}[t]
	\caption{Sampling Algorithm}
	\label{alg:generateSample}
	\SetKwFunction{generateSample}{}
	\SetKwProg{Fn}{generateSample}{:}{}
	\Fn{\generateSample{$\normalfont T$ }}
	{
		$t \sim p_{[0,T]}(t)$\;
		\For{$i=1:n_s$}
		{
			\If{$\lambda(\mathcal{F}[0,t])<\lambda(\mathcal{B}[t,T])$}
			{
				$\textbf{x}_\mathrm{cand} \leftarrow \mathrm{sampleUniform}(\mathcal{F}[0,t])$\;
				\If{$\normalfont \textbf{x}_\mathrm{cand} \in \mathcal{B}[t,T]$}
				{
					$\textbf{x}_\mathrm{rand} \leftarrow \textbf{x}_\mathrm{cand}$\;
					\KwRet $\textbf{x}_\mathrm{rand}$\;	
				}
			}
			\Else 
			{
				$\textbf{x}_\mathrm{cand} \leftarrow \mathrm{sampleUniform}(\mathcal{B}[t,T])$\;
				\If{$ \normalfont \textbf{x}_\mathrm{cand} \in \mathcal{F}[0,t]$}
				{
					$\textbf{x}_\mathrm{rand} \leftarrow \textbf{x}_\mathrm{cand}$\;
					\KwRet $\textbf{x}_\mathrm{rand}$\;	
				}
			}					
			
		}
		$\textbf{x}_\mathrm{rand} \leftarrow \mathrm{sampleUniform}(\mathcal{X})$\;			
		\KwRet $\textbf{x}_\mathrm{rand}$\;					
	}
\end{algorithm}
\DecMargin{.5em}
\IncMargin{.5em}
\begin{algorithm}[t]
	\caption{Vertex Inclusion Algorithm}
	\label{alg:includeVertex}
	\SetKwFunction{includeVertex}{}
	\SetKwProg{Fn}{includeVertex}{:}{}
	\Fn{\includeVertex{${\normalfont\textbf{v}},{\normalfont t },{\normalfont T}$ }}
	{
		\If{$t>T$}
		{
			\KwRet false \;
		}
		\ForEach{$\tau \in\{t+\delta,t+2\delta, \dots T \}$}
		{
			\If{$\normalfont \textbf{v}\in \mathcal{B}[t,\tau]$}
			{
				\KwRet true \;
			}
		}		
		\KwRet false\;				
	}
\end{algorithm}
\DecMargin{.5em}

\subsection{Vertex Inclusion Algorithm}

The vertex inclusion procedure, described in Algorithm~\ref{alg:includeVertex}, accepts a candidate vertex if it lies in $\Omega(T)$.
Consider a candidate vertex $\textbf{v}$ with cost-to-come $t$, i.e., the cost of trajectory from $\textbf{x}_\mathrm{s}$ to $\textbf{v}$ is $t$.
Since the cost-to-come is $t$, we have $\textbf{v} \in \mathcal{F}[0,t]$. 
Thus, if $\textbf{v} \in \mathcal{R}_b[t,T]$, then $\textbf{v} \in \Omega(T)$.
The proposed algorithm discretizes the interval $[t,T]$ with a step-size $\delta$.
A vertex is accepted if it lies in any $\mathcal{B}[t,\tau]$, for $\tau \in \{t+\delta,t+2\delta, \dots T \}$ (line 4-6).
The sets $\mathcal{B}[t,\tau]$ are again obtained from the stored library of reachable sets. 

In order to maintain the theoretical guarantees of TIE, an over-estimate of the solution cost $T$ is required.
This over-estimate can be obtained (and updated) after the planner discovers (and then improves) an initial, sub-optimal solution. 
Also, learning-based methods similar to \cite{chiang2019rl} can be used to obtain an estimate of the solution cost given a planning environment. 
In this work, the above algorithms are called only after an initial solution is discovered.
\begin{figure}
	\centering
	\includegraphics[width=0.6\columnwidth]{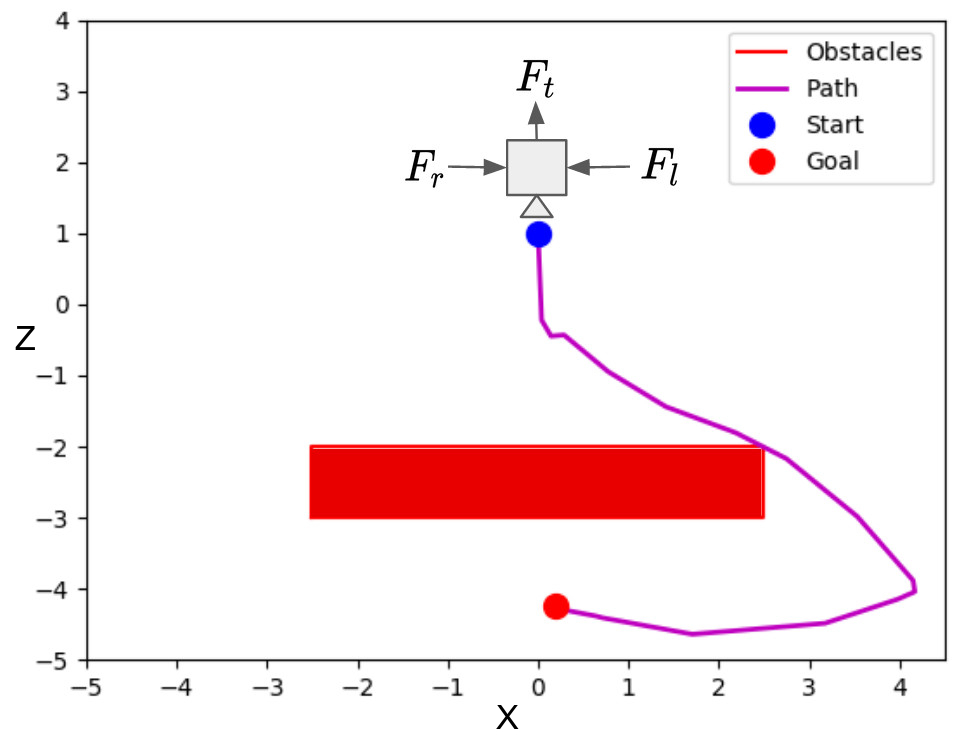}
	\includegraphics[width=0.6\columnwidth]{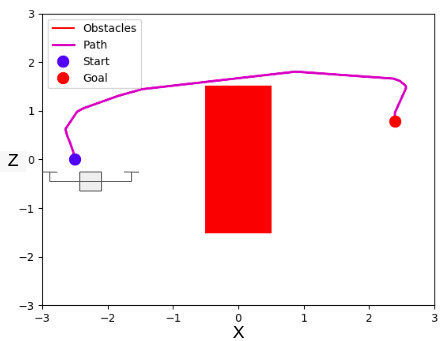}
	\caption{A schematic for the moon-lander robot (top) and quadrotor (bottom) simulation cases with sample solution paths found by the proposed algorithm after 40 sec of planning time.}
	\label{fig:moonlander_quadrotor}
\end{figure}

\section{NUMERICAL EXPERIMENTS}
\begin{figure*}
	\centering
	\includegraphics[width=0.64\columnwidth,height=0.37\columnwidth]{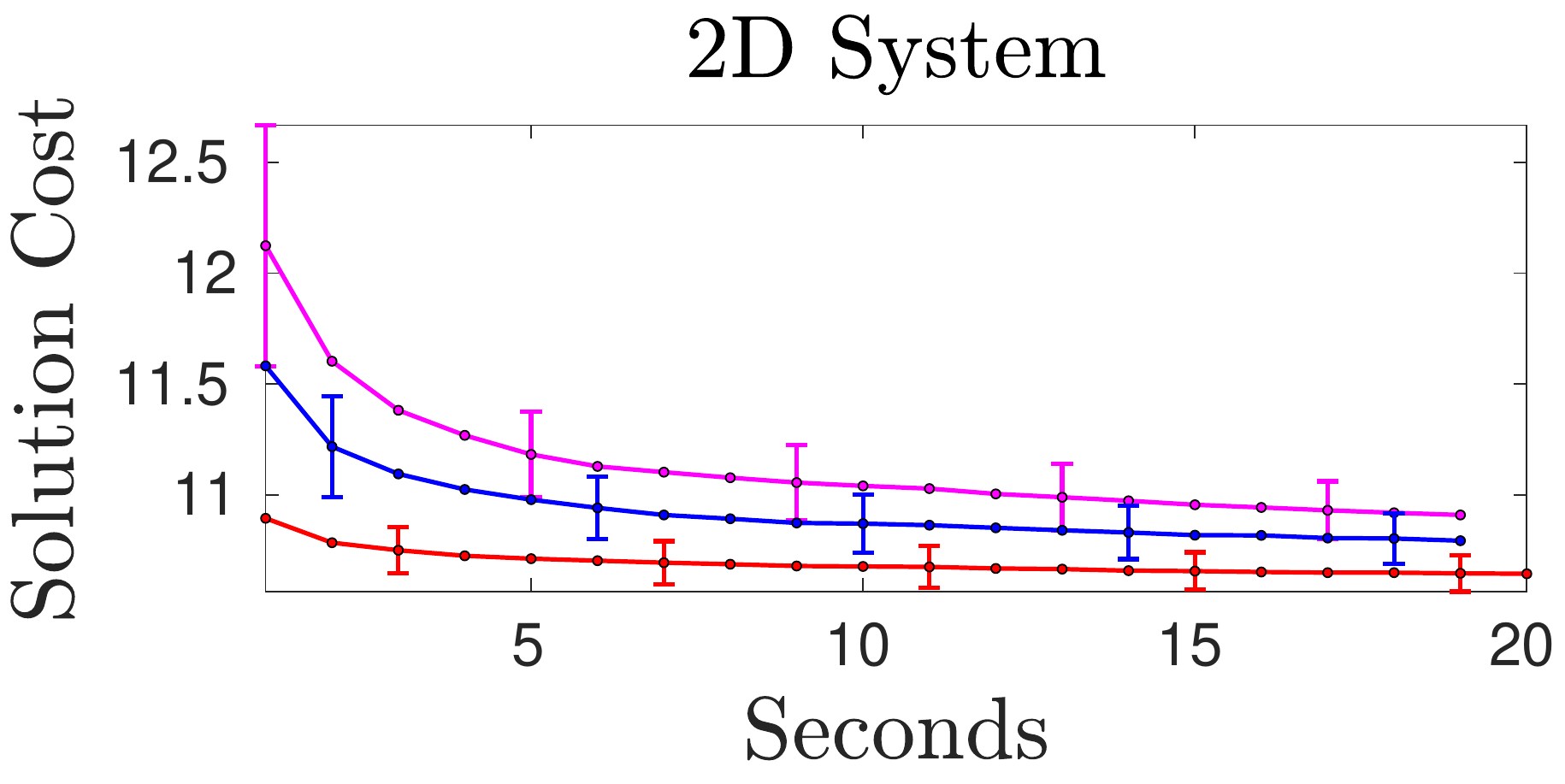}
	\includegraphics[width=0.64\columnwidth,height=0.37\columnwidth]{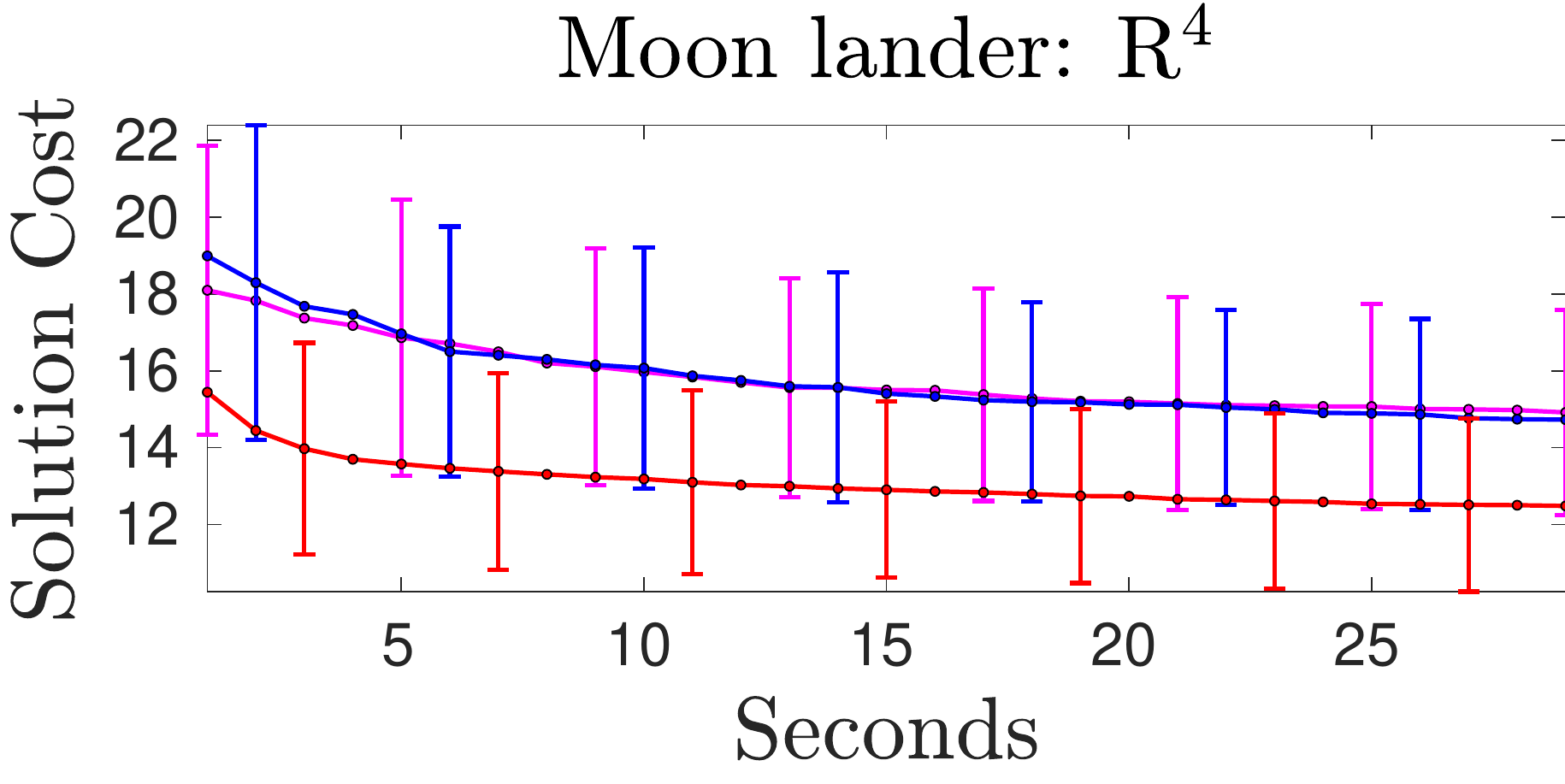}
	\includegraphics[width=0.64\columnwidth,height=0.37\columnwidth]{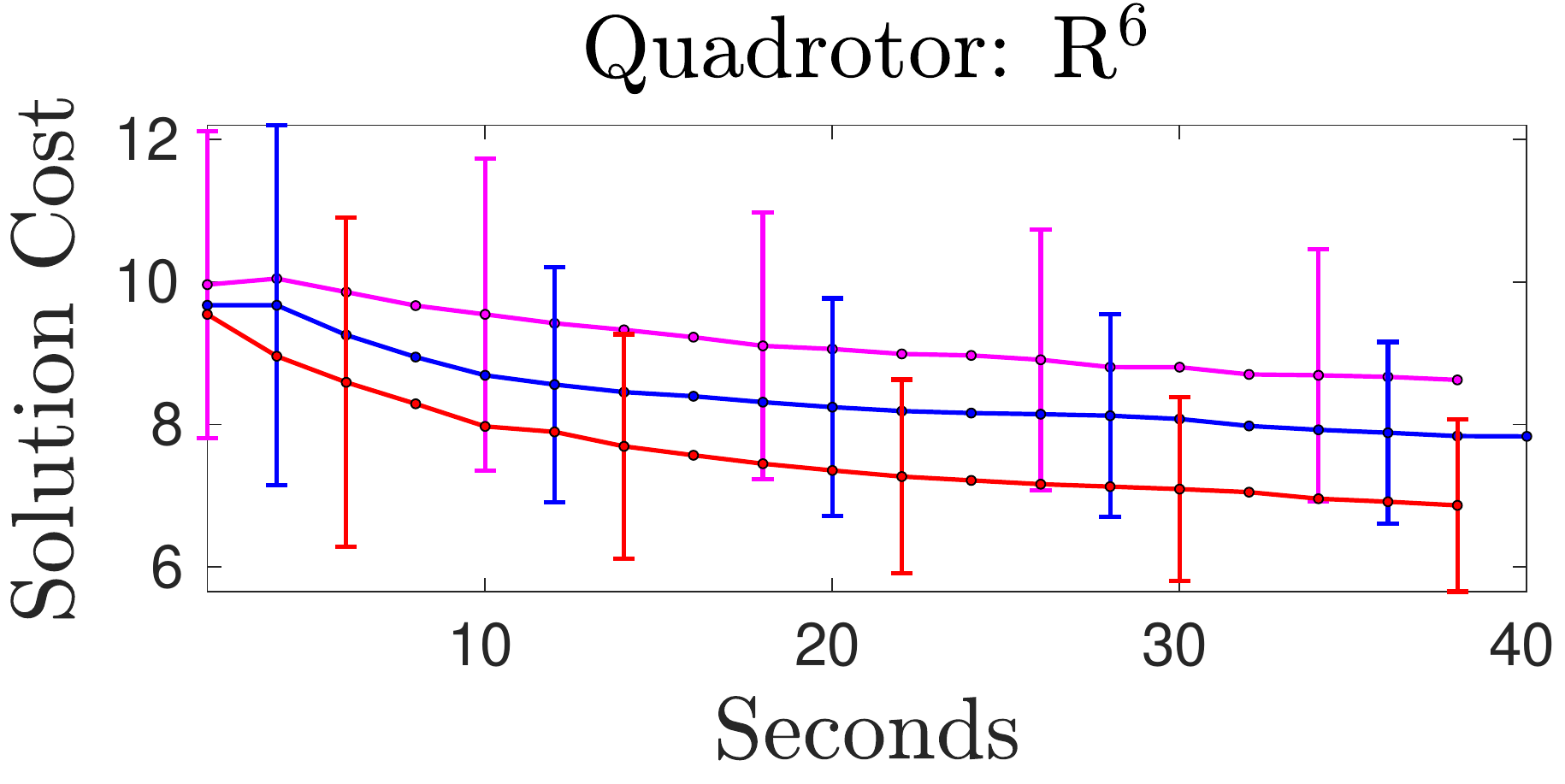}
	
	\includegraphics[width=0.64\columnwidth,height=0.37\columnwidth]{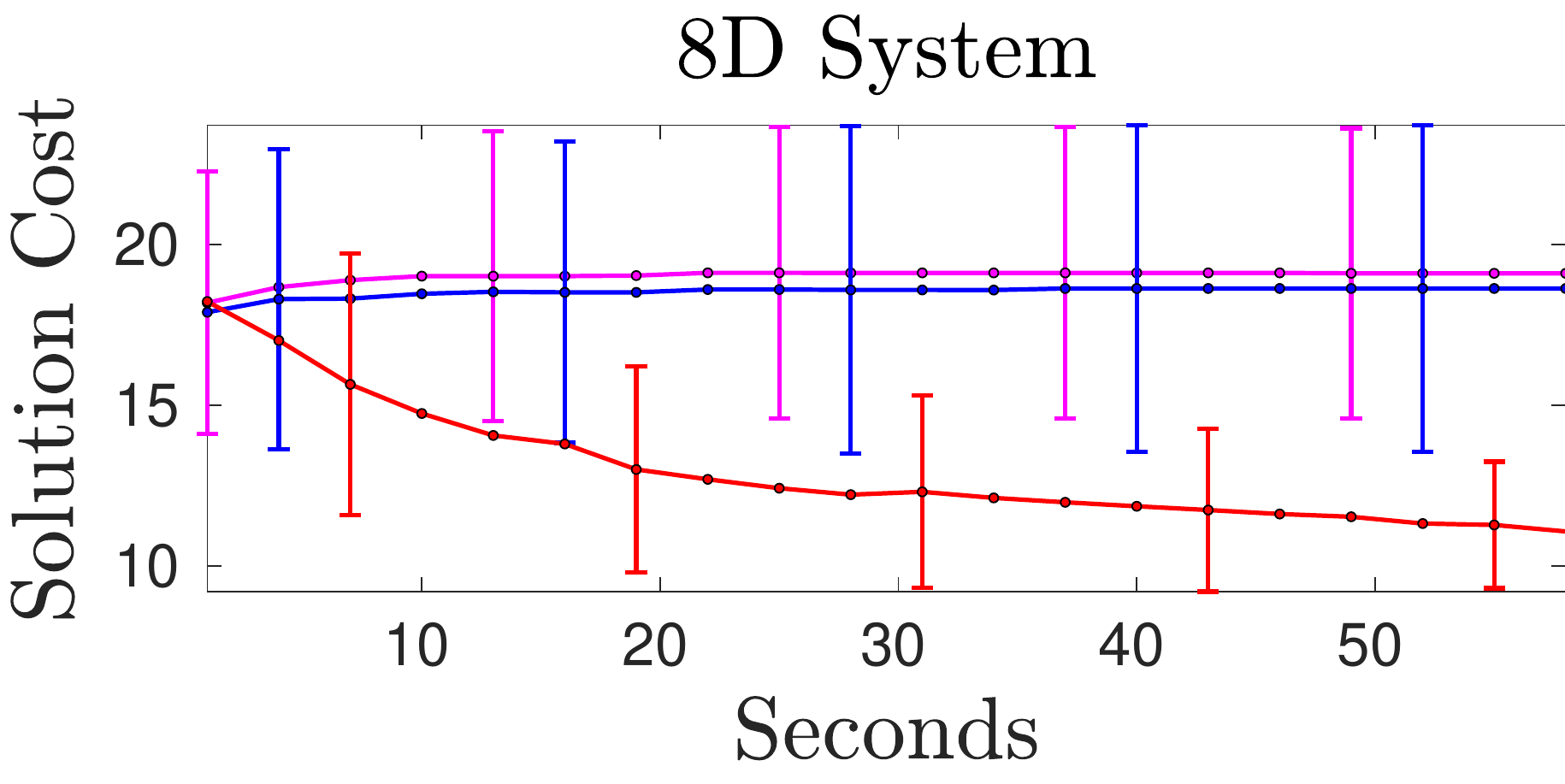}
	\includegraphics[width=0.64\columnwidth,height=0.37\columnwidth]{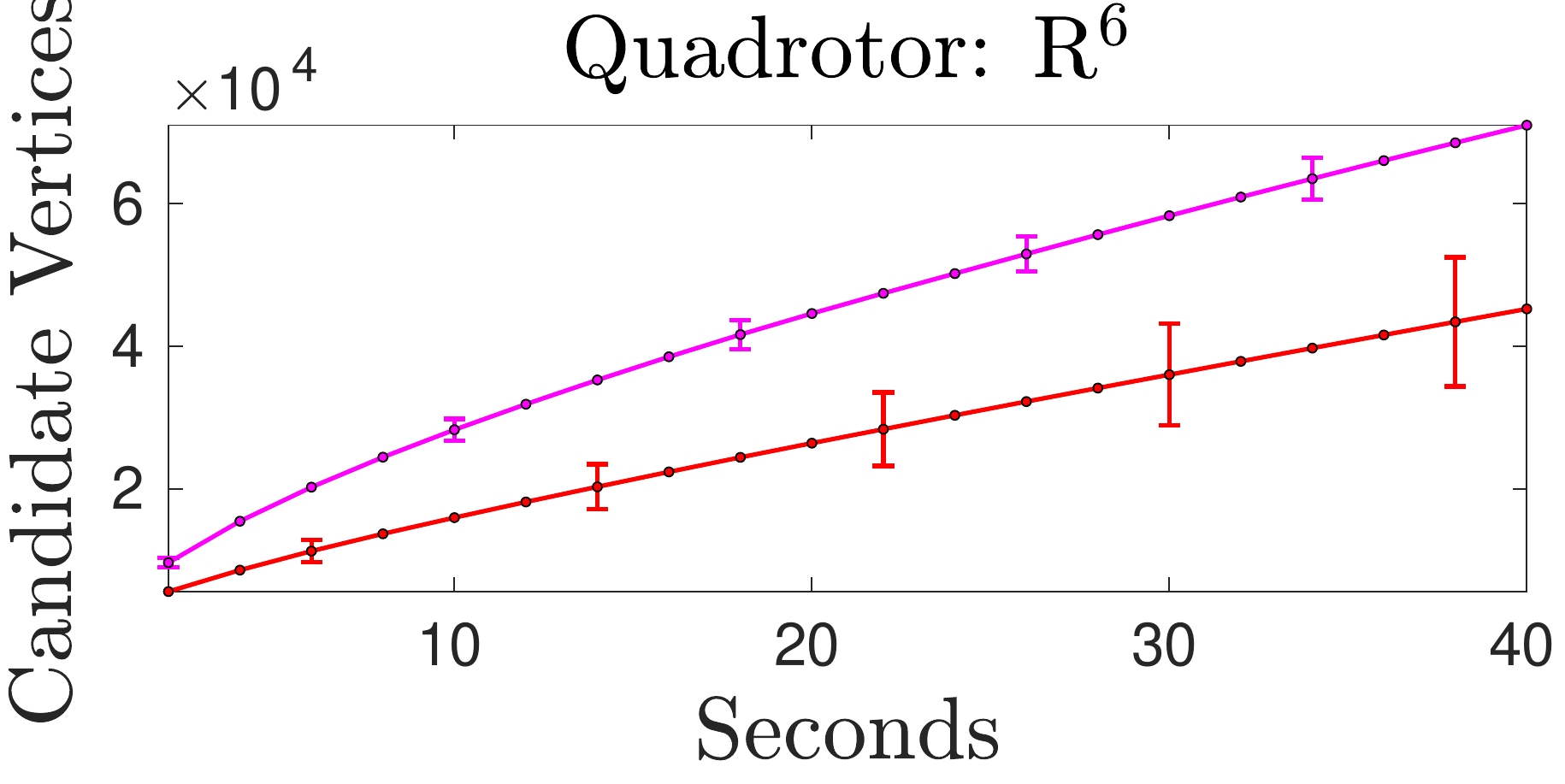}
	\includegraphics[width=0.64\columnwidth,height=0.37\columnwidth]{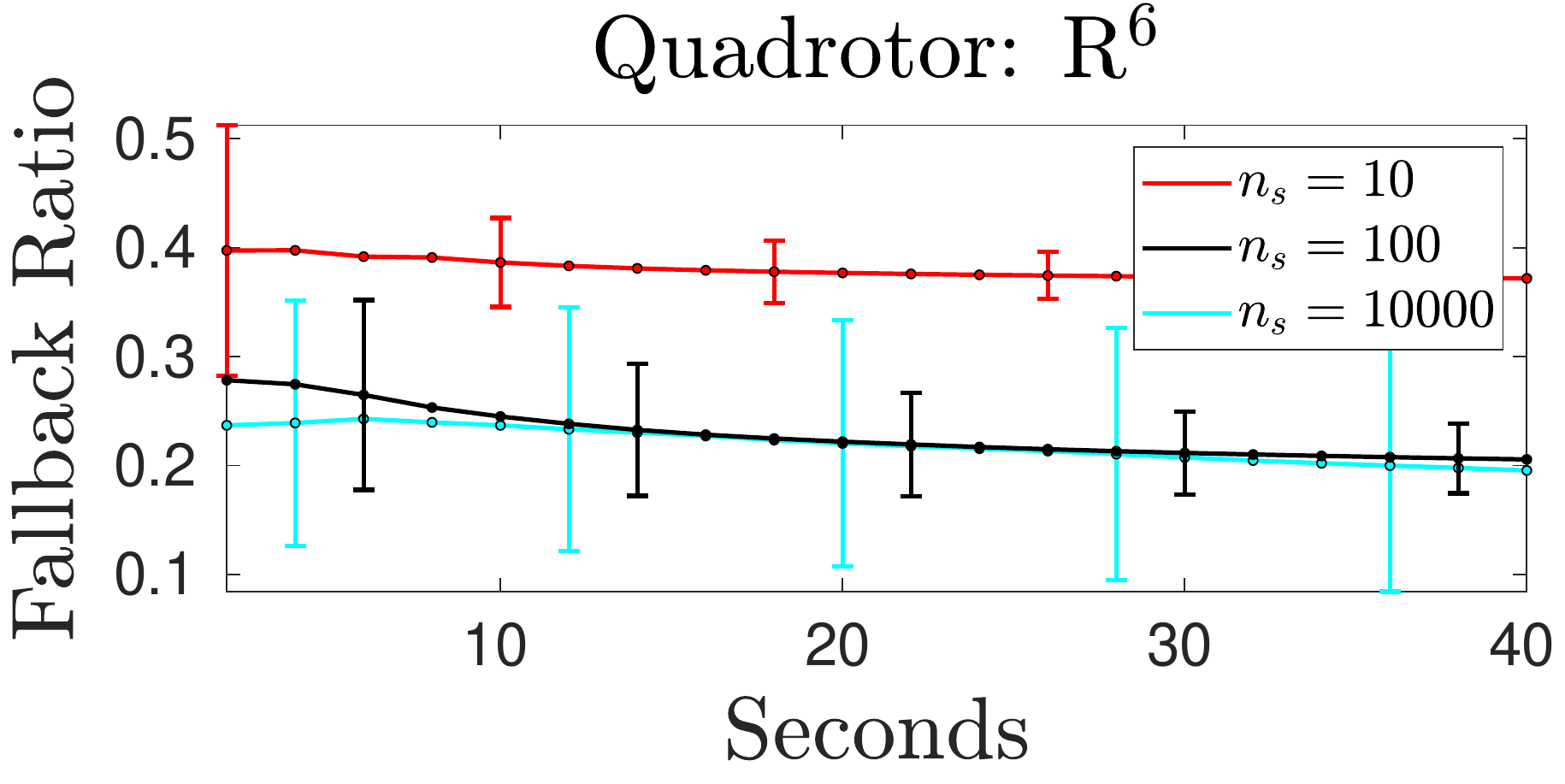}
	
	\includegraphics[width=0.64\columnwidth]{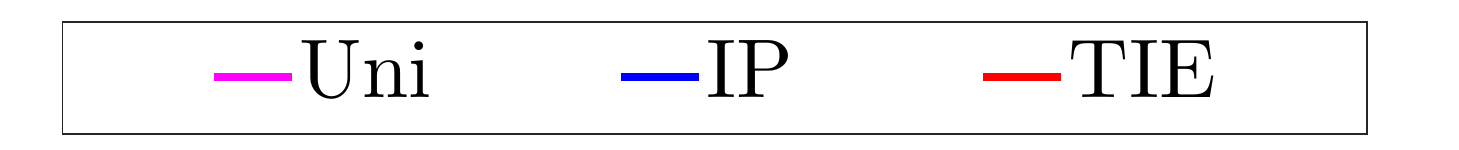}
	\caption{Convergence plots for the numerical experiments. Using the proposed TIE leads to a faster convergence in all cases (red plot). 
	The bottom middle figure illustrates number of candidate vertices generated using uniform and TIE exploration method.
	The bottom right figure plots the fallback ratio for different values of $n_S$.
	Solid lines indicate the value averaged over 100 trials and the error bars represent the standard deviation.}
	\label{fig:convergenceplots}
\end{figure*}
Benchmarking experiments were performed by pairing different exploration strategies with the SST planner~\cite{li2015sparse}.
All algorithms were implemented in C++ using the OMPL framework~\cite{sucan2012open}, and the tests were run using OMPL's standardized benchmarking tools~\cite{moll2015benchmarking}.
The data was recorded over 100 trials for all the cases on a 64-bit laptop PC with 16 GB RAM and an Intel i7 Processor, running Ubuntu 16.04 OS.
The performance of the proposed exploration strategy was benchmarked against uniform sampling (Uni) and uniform sampling combined with Informed propagation (IP).
Informed propagation essentially rejects expansion vertices with cost-to-come $t>T$, if there exists a sub-optimal solution with cost $T$.
If $t<T$, then forward propagation from the vertex is done for at most $T-t$ duration.
Thus, Informed Propagation (IP) prohibits exploration outside the set $\bigcup_{0 \leq \tau \leq T} \mathcal{F}[0,\tau]$.
The proposed Time-Informed exploration (TIE) algorithm uses the sampling and vertex inclusion procedures described in Algorithms~\ref{alg:generateSample} and \ref{alg:includeVertex} with $n_s=10$ and $\delta=0.1$.
The SST planner parameters, namely, the selection and pruning radius were set to standard OMPL values of 0.2 and 0.1 respectively. The $L_2$-norm was used as the distance function.
A description of different case-studies is given below. 

\noindent \textbf{2D System: }
Consider a 2D kino-dynamic system $\dot{\textbf{x}}=A_{2\times2}\textbf{x}+B_{2\times1}u $ with, $\textbf{x}=[x,\dot{x}]^\T$ and 
\begin{equation}
A_{2\times2}=
\begin{bmatrix}
0.0 & 0.5\\
-0.1 & 0.2\\
\end{bmatrix},
\quad
B_{2\times1}=
\begin{bmatrix}
0\\
1
\end{bmatrix}.
\end{equation}
The set-up of the planning problem is illustrated in Fig.~\ref{fig:toy2d}, with $\textbf{x}_\mathrm{s}=[-3 ~0]^\T,\textbf{x}_\mathrm{g}=[3 ~0]^\T, \mathcal{X}_\mathrm{g}= \mathcal{E}(\textbf{x}_\mathrm{g}, 0.25~ \mathrm{I}_2),u \in [-0.5 ~0.5]$. Here, $\mathrm{I}_2$ represents the $2 \times 2$ identity matrix. 

\noindent \textbf{8D System:} The 2D system described above is extended to a 8D system $\dot{\textbf{x}}=A_{8\times8}\textbf{x}+B_{8 \times 4}\textbf{u}$, with $\textbf{x}=[x_1 ~\dot{x}_1 ~x_2 ~\dot{x}_2 ~x_3 ~\dot{x}_3 ~x_4 ~\dot{x}_4]^\T, \textbf{u}=[u_1 ~u_2 ~u_3 ~u_4]^\T$ and 
\begin{equation}
\begin{aligned}
A_{8\times8} &= 
\mathrm{blkdiag}[A_{2 \times 2}, A_{2 \times 2},A_{2 \times 2},A_{2 \times 2}],\\
B_{8 \times 4} &= \mathrm{blkdiag}[B_{2 \times 1}, B_{2 \times 1},B_{2 \times 1},B_{2 \times 1}].
\end{aligned}
\end{equation}
The single obstacle in 2D case was extended to 8D by adding a length of 2 units symmetrically in the extra dimensions.
Also, $\textbf{x}_\mathrm{s}=[-2 ~0 ~0 ~0 ~0 ~0 ~0 ~0]^\T,\textbf{x}_\mathrm{g}=[2 ~0 ~0 ~0 ~0 ~0 ~0 ~0]^\T, \mathcal{X}_\mathrm{g}= \mathcal{E}(\textbf{x}_\mathrm{g}, \mathrm{I}_8),u_i \in [-1 ~1], i \in \{ 1,2,3,4 \}$.

\noindent \textbf{Moon-lander Robot:} 
A simplified version of a planar "moon-lander robot" is illustrated in Fig.~\ref{fig:moonlander_quadrotor}. The robot has three thrusters $F_l,F_r$ and $F_t$ acting in the left, right and up direction respectively. In the absence of upwards thrust, the robot falls under gravity. The dynamics of the robot is assumed to be as follows.
\begin{equation}
\frac{\dint}{\dint t}
\begin{bmatrix}
x\\ z\\ \dot{x}\\ \dot{z}
\end{bmatrix}
=
\begin{bmatrix}
0 & 0 & 1 & 0\\
0 & 0 & 0 & 1\\
0 & 0 & 0 & 0\\
0 & 0 & 0 & 0
\end{bmatrix}
\begin{bmatrix}
x\\ z\\ \dot{x}\\ \dot{z}
\end{bmatrix}
+
\begin{bmatrix}
0 & 0 & 0\\
0 & 0 & 0\\
-2& 1 & 0\\
0& 0 &1
\end{bmatrix}
\begin{bmatrix}
F_l\\ F_r \\F_t
\end{bmatrix}.
\end{equation}
The start, goal and admissible control space were set as follows: $\textbf{x}_\mathrm{s}=[0 ~1 ~0 ~-2]^\T,\textbf{x}_\mathrm{g}=[0 ~-4 ~0 ~0]^\T, \mathcal{X}_\mathrm{g}= \mathcal{E}(\textbf{x}_\mathrm{g}, 0.25~ \mathrm{I}),F_l \in [0 ~1],F_r \in [0 ~1],F_t \in [-2 ~2]$. The objective is to land the robot in time-optimal fashion.

\noindent \textbf{Planar Quadrotor model:}
A linearized quadrotor model for longitudinal flight based on \cite{michael2010grasp} can be written as $\dot{\textbf{x}}=A_{6\times6}\textbf{x}+B_{6 \times 2}\textbf{u}$, with $\textbf{x}=[x ~z ~u ~w ~q  ~\theta]^\T, \textbf{u}=[f_t ~\tau_y]^\T$ and 
\begin{equation}
A_{6\times6}
=
\begin{bmatrix}
0 & 0 & 1 & 0 & 0 &0\\
0 & 0 & 0 & 1 & 0 & 0\\
0 & 0 & 0 & 0 & 0 &-g\\
0& 0 & 0 & 0 & 0 & 0\\
0& 0 & 0 & 0 & 0 & 0\\
0 & 0 & 0 & 0 & 1 & 0
\end{bmatrix}
B_{6 \times 2} 
=
\begin{bmatrix}
0 & 0\\
0 & 0\\
0 & 0\\
1/m & 0\\
0 & 1/I_y\\
0 & 0
\end{bmatrix}
\end{equation}
The start, goal and admissible control space were set as follows: $\textbf{x}_\mathrm{s}=[-2.5 ~0 ~0 ~0 ~0 ~0]^\T,\textbf{x}_\mathrm{g}=[2.5 ~0 ~0 ~0 ~0 ~0]^\T,\mathcal{X}_\mathrm{g}= \mathcal{E}(\textbf{x}_\mathrm{g}, \mathrm{I}), f_t \in [-1 ~1],\tau_y \in [-1 ~1]$. 
The set-up for time-optimal planning problem is shown in Fig.~\ref{fig:moonlander_quadrotor}. 

The results of the numerical simulations are illustrated in Fig.~\ref{fig:convergenceplots}. 
It can be seen that Informed Propagation (blue) performs better than the na\"{\i}ve uniform exploration (magenta). 
However, using the proposed TIE strategy, a combination of Algorithm~\ref{alg:generateSample} and \ref{alg:includeVertex}, outperforms the other methods in all cases. 
Note that for a planner such as SST, the sampling procedure influences the vertex to be selected for forward propagation. 
Generating random samples $\textbf{x}_\mathrm{cand} \in \Omega(T)$ biases the selection of vertices in the TIS for expansion.
After a vertex is selected, expansion is performed by forward propagating the system dynamics to generate a new candidate vertex $\textbf{v}$. 
The vertex inclusion algorithm then ensures that the candidate vertex $\textbf{v} \in \Omega(T)$. 
Thus, the combination of the proposed sampling and inclusion algorithm avoids redundant exploration focuses search, and leads to a faster convergence in all cases. 
In order to study the computational cost incurred by the TIE procedure, the quadrotor simulation was run without obstacles. 
Compared to uniform sampling, TIE generates a lower number of feasible, candidate vertices for inclusion in the planner tree, as illustrated in Fig.~\ref{fig:convergenceplots}.
Future work will explore leveraging GPUs and operating on batch of samples and reachable sets in parallel. 
Parameter $n_s$ controls the maximum number of attempts made to generate a new sample $\textbf{x}_\mathrm{cand} \in  \Omega(T)$ before a uniform random sample is returned.
In order to analyze the effect of $n_s$, the "fallback ratio" is defined as,
\begin{displaymath}
    \text{Fallback Ratio}= \frac{\text{Number of Fallbacks to Uniform Sampling}}{\text{Number of Calls to the TIE Sampler}}
\end{displaymath}
The fallback ratio was found to be negligible for the lower dimensional 2D system. 
It is relatively higher for the 6D quadrotor simulation run in a no-obstacle environment (see Fig.~\ref{fig:convergenceplots}). 
This ratio can be decreased by increasing $n_s$. 
However, a large value of $n_s$ corresponds to
a larger amount of computations invested in the sampling procedure, which can adversely impact the convergence of solution cost.
Note that while the sampling algorithm may return $\textbf{x}_\mathrm{cand}$ outside the TIS if $n_s$ attempts are exhausted, the vertex inclusion procedure ensures that a candidate vertex $\textbf{v}$ is incorporated in the planner tree only if it lies in the TIS.

\section{CONCLUSION}

In this work, we use ideas from reachability analysis to define a "Time-Informed Set", to focus exploration after an initial solution is found.
We prove that exploring the TIS is a necessary condition to improve the current solution.
The proposed method can be applied to a variety of systems for which an efficient local steering module may not be available, but (over-)approximations of the reachable sets can be constructed.

It should be noted that  the $L_2$-Informed set is \textit{sharp} \cite{gammell2018informed}, i.e., it  uses a heuristic estimate which gives the exact cost-to-come and cost-to-go for any point in the absence of obstacles. The TIS is not so, as it is constructed using \textit{over-approximations} of the reachable sets. 
Hence, finding tight approximations of the reachable sets is critical for the efficacy of the proposed approach.


In order to apply TIE for sampling-based planning, the reachability library needs to constructed offline. 
Creating, storing and accessing this library should be computationally efficient  for  higher  dimensional  systems to be of use in practice.
The ellipsoidal reachable sets used in this work satisfy these criteria. 
The HJB reachability toolboxes \cite{bansal2017hamilton} can be potentially used to create this library for a general non-linear systems.
These frameworks solve the value function PDE by discretizing the state space. However, 
the computational cost of these methods scale exponentially with the dimension.
In order to address this curse of dimensionality, application of deep-learning frameworks for reachability, such as \cite{allen2014machine}, \cite{bansal2020deepreach}, can be explored.
Recent works such as DeepReach~\cite{bansal2020deepreach} avoid gridding the state space and use deep neural networks (DNN) to learn a parameterized approximation of the value function.
These DNNs can be stored and used to classify or generate new samples in the TIS.

\vspace*{1ex}

\textbf{Acknowledgements:} The authors thank Dipankar Maity and Kelsey Hawkins for insightful discussions on this topic. This work has been supported by NSF awards IIS-1617630 and IIS-2008686.

\bibliographystyle{IEEEtran}
\bibliography{refs}	
	
\end{document}